\documentclass{article}

\usepackage{microtype}
\usepackage{graphicx}
\usepackage{subfigure}
\usepackage{booktabs} 
\usepackage[hidelinks]{hyperref}
\usepackage[margin=1in]{geometry}

\usepackage{amsmath,amssymb,amsthm}
\usepackage{dsfont}
\usepackage[framemethod=TikZ]{mdframed}

\PassOptionsToPackage{usenames,dvipsnames}{xcolor}
\usetikzlibrary{arrows,positioning,automata,shapes,calc,patterns}

\newtheorem{theorem}{Theorem}
\newtheorem{definition}{Definition}
\newtheorem{corollary}{Corollary}
\newtheorem{proposition}{Proposition}

\newtheorem{remark}{Remark}

\DeclareMathOperator*{\argmin}{arg\,min}

\DeclareMathOperator{\trace}{trace}

\newcommand{\D}{\mathcal{D}}
\newcommand{\U}{\mathcal{U}}
\newcommand{\J}{\mathcal{J}}
\newcommand{\Z}{\mathcal{Z}}
\renewcommand{\P}{\mathbb{P}}
\newcommand{\E}{\mathbb{E}}

\renewcommand{\c}{\mathrm{c}}
\renewcommand{\d}{\mathrm{d}}

\begin{document}

\title{Modelling and Quantifying Membership Information Leakage in Machine Learning}

\author{Farhad Farokhi\thanks{F. Farokhi is with the University of Melbourne, Australia. Email: farhad.farokhi@unimelb.edu.au}\;  and Mohamed Ali Kaafar\thanks{M. A. Kaafar is with the CSIRO's Data61 and Macquarie University, Australia.}}

\maketitle

\begin{abstract}
Machine learning models have been shown to be vulnerable to membership inference attacks, i.e., inferring whether individuals' data have been used for training models. The lack of understanding about factors contributing success of these attacks motivates the need for modelling membership information leakage using information theory and for investigating properties of machine learning models and training algorithms that can reduce membership information leakage. We use conditional mutual information leakage to measure the amount of information leakage from  the trained machine learning model about the presence of an individual in the training dataset. We devise an upper bound for this measure of information leakage using Kullback--Leibler divergence that is more amenable to numerical computation. We prove a direct relationship between the Kullback--Leibler membership information leakage and the probability of success for a hypothesis-testing adversary examining whether a particular data record belongs to the training dataset of a machine learning model. We show that the mutual information leakage is a decreasing function of the training dataset size and the regularization weight. We also prove that, if the sensitivity of the machine learning model (defined in terms of the derivatives of the fitness with respect to model parameters) is high, more membership information is potentially leaked. This illustrates that complex models, such as deep neural networks, are more susceptible to membership inference attacks in comparison to simpler models with fewer degrees of freedom. We show that the amount of the membership information leakage is reduced by $\mathcal{O}(\log^{1/2} (\delta^{-1})\epsilon^{-1})$ when using Gaussian  $(\epsilon,\delta)$-differentially-private additive noises.  
\end{abstract}

\section{Introduction}
In the era of big data, advanced machine learning techniques enable accurate data analytic for various application domains. This has incentivized commercial access to machine learning models by third-party users, e.g., machine learning as a service provided by data giants, such as Google and Amazon, allows companies to train  models on their data and to sell access to these models. Although commercially attractive, these services open the door to data theft and privacy infringements. An example of this is membership inference attack by which an adversary attempts at inferring whether individuals' data is used for training a machine learning model~\cite{10shokri2017membership}. 

In this paper, we propose membership information leakage metrics to investigate the reasons behind the success of membership inference attacks. We use conditional \textbf{mutual information leakage} to measure the amount of information leakage from  the trained machine learning model about the presence of an individual in the training dataset. We find an upper bound for this measure of information leakage using Kullback--Leibler divergence between the distribution of the machine learning models in the presence of a particular data record and in the absence of that data record. Following this, we define \textbf{Kullback--Leibler membership information leakage}. Using the Le Cam's inequality~\cite{Yu1997} and the Pinsker's inequality~\cite{massart2007concentration}, we show that this measure bounds the \textbf{probability of success of any adversary trying to determine if a particular data record belongs to the training dataset of a machine learning model}. This provides an information-theoretic interpretation for our choices of membership  information leakage metrics.

We use the developed measures of membership information leakage to investigate factors behind the success of membership inference attacks. We first prove that the \textbf{amount of the membership information leakage is a decreasing function of the training dataset size}. This signifies that, by using a larger training dataset, the model is less over-fit to the training dataset and therefore it is harder to distinguish the training data from the rest. This particular result is applicable to general machine learning models ranging from linear regression to deep neural networks as it does not require convexity. By focusing on convex machine learning problems, we investigate other important factors behind the success of membership inference attacks. We prove that \textbf{regularization reduces the amount of membership information leakage}. This can again be attributed to that increasing the importance of the regularization reduces over-fitting and is therefore an important tool for combating membership inference attacks. Then, we define sensitivity of machine learning models by bounding variations of the model fitness  across all the data entries and model parameters. Following this, we prove that \textbf{less membership information is leaked if the training dataset is more sensitive}. This can illustrate that complex models, such as deep neural networks, are more susceptible to membership inference attacks in comparison to simpler models with fewer degrees of freedom.

Finally, we study the effect of additive noise on the success of membership inference attacks by quantifying the amount of decrease in the membership information leakage caused by additive Gaussian noise. We particularly prove that \textbf{membership information leakage reduces by $\mathcal{O}(\log^{1/2}(\delta^{-1})\epsilon^{-1})$ when using $(\epsilon,\delta)$-differentially-private additive Gaussian noises}, following the Gaussian mechanism in~\cite[Theorem~A.1]{dwork2014algorithmic}.

\section{Related Work}
Membership inference attacks, a class of adversarial inference algorithms designed to distinguish data used for training a machine learning model, have recently gained much attention~\cite{10shokri2017membership, 18truex2018towards,salem2018ml, sablayrolles2019white}. These attacks have been deployed on various machine learning models; see, e.g.,  \cite{10shokri2017membership, 17hayes2019logan, chen2019gan, hilprecht2019monte,liu2019socinf, backes2016membership}. The success of the attacks is often attributed to that a machine learning model behaves differently on the training dataset and the test dataset, e.g., it shows higher confidence on the training dataset due to an array of reasons, such as over-fitting. 

Many defence mechanisms have been proposed against membership inference attacks. A game-theoretic approach is proposed in~\cite{nasr2018machine}, where a regularization term using the accuracy of membership inference attacks is incorporated when training machine learning models. Others have introduced indistinguishability for membership inference attacks as an estimate of the  discrimination of the model on training and test datasets~\cite{yaghini2019disparate}. Alternatively, it has been suggested that we can counter membership inference attacks by reducing over-fitting~\cite{yeom2018privacy}. Membership inference attacks are shown to work better on certain subgroups of the population, e.g., underrepresented minorities, resulting in disparate vulnerability~\cite{yaghini2019disparate}. Furthermore, success of membership inference attack may not predict success of attribute inference attacks with only access to  partial view of data records~\cite{zhao2019inferring}. Another approach is to use differentially-private machine learning at the cost of significantly reducing the utility~\cite{12rahman2018membership, leino2019stolen}.  However, none of these capture the possibly many factors contributing to the success of membership inference attacks.

This motivates taking a deeper look at the factors behind the success of membership inference attacks using information-theoretic membership information leakage metrics. This is the topic of this paper. 

Finally, we would like to point out recent results exploring differential privacy and mutual information, e.g., see~\cite{cuff2016differential,wang2016relation}. Although these results provide important insights into information-theoretic guarantees of differential privacy, they are far from the context of this paper and do not consider membership inference attacks.

\section{Membership Information Leakage}
Consider all possible data records in a universe $\U: =\{(x_i,y_i)\}_{i=1}^{N}\subseteq \mathbb{R}^{p_x}\times\mathbb{R}^{p_y}$ in which $x_i$ and $y_i$ denote inputs and outputs, respectively. Note that the data universe is not necessarily finite; $N$ can be infinite. A machine learning algorithm only has access to $n$ entries from this data universe. We denote this by the private training dataset $\D\subseteq \U$. Hence, the size of the training dataset is $|\mathcal{D}|=n<N$. Let $(z_i)_{i=1}^N\in\{0,1\}^N$ be such that $z_i=1$ if $(x_i,y_i)\in\D$ and $z_i=0$ otherwise. Let $(z_i)_{i=1}^N\in\{0,1\}^N$ be uniformly selected at random from  $\Z:=\{(z_i)_{i=1}^N\in\{0,1\}^N|\sum_{i=1}^N z_i=n\}$. This implies that any record in the data universe $\U$ is equally likely to be part of the training dataset $\D$. This is a common assumption in machine learning~\cite{anthony2009neural} and membership inference~\cite{salem2018ml}. 

Consider a generic supervised machine learning problem with the aim of training a model $\mathfrak{M}(\cdot;\theta):\mathbb{R}^{p_x}\rightarrow \mathbb{R}^{p_y}$ to capture the relationship between inputs and outputs in the training dataset $\D$ by solving the optimization problem in
\begin{align}\label{eqn:ML}
\theta_{\c}^*\in\argmin_{\theta\in\Theta_{\c}} \; f(\theta,\D),
\end{align}
with
\begin{align*}
f(\theta,\D):=\lambda g(\theta)+\frac{1}{|\D|} \sum_{(x,y)\in\D}\ell(\mathfrak{M}(x;\theta),y),
\end{align*}
where $\ell(\mathfrak{M}(x;\theta),y)$ is a loss function capturing the ``closeness'' of the outcome of the trained ML model $\mathfrak{M}(x;\theta)$ to the actual output $y$,  $g(\theta)$ is a regularizing term, $\lambda\geq 0$ is a weight balancing between the loss function and the regularizing term, and $\Theta_{\c}\subseteq\mathbb{R}^{p_\theta}$ denotes the set of feasible models. Computers only have a finite precision. Therefore, in practice, we can only compute the optimal model to a finite precision as
\begin{align}\label{eqn:ML_d}
\theta_{\d}^*:=\Pi_{\Theta_{\d}}[\theta_{\c}^*],
\end{align}
where $\Pi_{\Theta_{\d}}(\cdot)$ denotes projection into the finite set $\Theta_{\d}\subset\Theta_{\c}$. The set $\Theta_{\d}$ can, for instance, be the intersection of the set of feasible models $\Theta_{\c}$ and the set of rational numbers modeled by the floating point number representation of the utilized computing unit. 

For an arbitrary $(x_i,y_i)\in\U$, an adversary is interested in inferring whether $(x_i,y_i)$ belongs to $\D$ or not based on the knowledge of $(x_i,y_i)$ and $\theta^*_\d$. We use conditional mutual information between $\theta^*_\d$ and $z_i$  as a measure of how much information regarding $z_i$, i.e., whether $(x_i,y_i)$ belongs to $\mathcal{D}$ or not, is leaked through $\theta^*_\d$. 

\begin{mdframed}[backgroundcolor=black!10,rightline=false,leftline=false,topline=false,bottomline=false,roundcorner=2mm] 
	\vspace{-.1in}
	\begin{definition}[Mutual Membership Information Leakage]
	We measure membership information leakage in machine learning by
		\begin{align*}
		\rho_{\mathrm{MI}}(\theta^*_\d):=&I(\theta^*_\d;z_i|x_i,y_i)\\
		=& \E\Bigg\{\log\Bigg(\frac{p(\theta^*_\d,z_i|x_i,y_i)}{
			p(\theta^*_\d|x_i,y_i)p(z_i|x_i,y_i)}\Bigg)\Bigg\}.
		\end{align*}
		Similarly, we can define $\rho_{\mathrm{MI}}(\theta^*_\c)$. Data processing inequality implies that  $\rho_{\mathrm{MI}}(\theta^*_\d)\leq \rho_{\mathrm{MI}}(\theta^*_\c)$.
	\end{definition}
\end{mdframed}

\begin{remark}[Average vs Worst-Case]
Mutual information 1 relies on expectation. Therefore, if the measure of the information leakage is non-zero, it means that a percentage of the population is vulnerable to membership inference attacks. Membership inference attacks are however more effective on samples that are distinct from the others~\cite{yaghini2019disparate}. This model captures those cases so long as the distinct population appears with a non-zero probability. That being said, as the experimental results show in Section~\ref{sec:numerical}, there are always outliers in the success of membership inference attacks. For investigating those cases, we need to use max entropy, i.e., Renyi entropy of order infinity. This is left as future work.
\end{remark}

We can rewrite the conditional mutual information as
\begin{align*}
I(\theta^*_\d;z_i|x_i,y_i)
=& \E\Bigg\{\log\Bigg(\frac{p(\theta^*_\d,z_i|x_i,y_i)}{
	p(\theta^*_\d|x_i,y_i)p(z_i|x_i,y_i)}\Bigg)\Bigg\}\\
=& \E\Bigg\{\log\Bigg(\frac{p(\theta^*_\d|z_i,x_i,y_i)}{
	p(\theta^*_\d|x_i,y_i)}\Bigg)\Bigg\}.
\end{align*}
where $p(\theta^*_\d|x_i,y_i)=p(\theta^*_\d|x_i,y_i,z_i=0)\P\{z_i=0\}
+p(\theta^*_\d|x_i,y_i,z_i=1)\P\{z_i=1\}.$
In what follows, let $p_0(\theta^*_\d)=p(\theta^*_\d|x_i,y_i,z_i=0)$,  $p_1(\theta^*_\d)=p(\theta^*_\d|x_i,y_i,z_i=1)$, $\alpha_0=\P\{z_i=0\}=n/N$, and $\alpha_1=\P\{z_i=1\}=1-\alpha_0.$ Therefore, 
\begin{align}
I(\theta^*_\d&;z_i|x_i,y_i)\nonumber\\
=& \E\Bigg\{\log\Bigg(\frac{p(\theta^*_\d|z_i,x_i,y_i)}{
	\alpha_0p_0(\theta^*_\d)+\alpha_1p_1(\theta^*_\d)}\Bigg)\Bigg\}\nonumber
\\
=&\E\Bigg\{\alpha_0\log\Bigg(\frac{p_0(\theta^*_\d)}{
	\alpha_0p_0(\theta^*_\d)+\alpha_1p_1(\theta^*_\d)}\Bigg)\nonumber +\alpha_1\log\Bigg(\frac{p_1(\theta^*_\d)}{
	\alpha_0p_0(\theta^*_\d)+\alpha_1p_1(\theta^*_\d)}\Bigg)\Bigg\}\nonumber
\\
=&\E\{\alpha_0D_{\mathrm{KL}}(p_0(\theta^*_\d)||
	\alpha_0p_0(\theta^*_\d)+\alpha_1p_1(\theta^*_\d)) +\alpha_1D_{\mathrm{KL}}(p_1(\theta^*_\d)
	\alpha_0p_0(\theta^*_\d)+\alpha_1p_1(\theta^*_\d))\}.
\label{eqn:MI_reduced_KL}	
\end{align}
From~\eqref{eqn:MI_reduced_KL}, we can develop an upper bound for $I(\theta^*_\d;z_i|x_i,y_i)$ based on the convexity of the Kullback--Leibler divergence $D_{\mathrm{KL}}(p||q)$ with respect to $q$. This bound is easier to numerically compute and is thus used in our numerical evaluations. Note that 
\begin{align*}
D_{\mathrm{KL}}(p_0(\theta^*_\d)||
\alpha_0p_0(\theta^*_\d)+\alpha_1p_1(\theta^*_\d))
\leq &\alpha_0D_{\mathrm{KL}}(p_0(\theta^*_\d)||
p_0(\theta^*_\d))
+\alpha_1D_{\mathrm{KL}}(p_0(\theta^*_\d)||
p_1(\theta^*_\d))\\
=&\alpha_1D_{\mathrm{KL}}(p_0(\theta^*_\d)||
p_1(\theta^*_\d)),
\end{align*}
and 
\begin{align*}
D_{\mathrm{KL}}(p_1(\theta^*_\d)||
\alpha_0p_0(\theta^*_\d)+\alpha_1p_1(\theta^*_\d))
\leq &\alpha_0D_{\mathrm{KL}}(p_1(\theta^*_\d)||
p_0(\theta^*_\d))+\alpha_1D_{\mathrm{KL}}(p_1(\theta^*_\d)||
p_1(\theta^*_\d))\\
=&\alpha_0D_{\mathrm{KL}}(p_1(\theta^*_\d)||
p_0(\theta^*_\d)).
\end{align*}
These inequalities imply that
\begin{align*}
I(\theta^*_\d;z_i|x_i,y_i)
\leq &\alpha_0\alpha_1\E\{D_{\mathrm{KL}}(p_0(\theta^*_\d)||
p_1(\theta^*_\d)) +D_{\mathrm{KL}}(p_1(\theta^*_\d)||
p_0(\theta^*_\d))\}\\
\leq &\frac{1}{4}\E\{D_{\mathrm{KL}}(p_0(\theta^*_\d)||
p_1(\theta^*_\d)) +D_{\mathrm{KL}}(p_1(\theta^*_\d)||
p_0(\theta^*_\d))\}.
\end{align*}
This derivation motivates the introduction of another measure for membership information leakage in machine learning using the Kullback--Leibler divergence. 

\begin{mdframed}[backgroundcolor=black!10,rightline=false,leftline=false,topline=false,bottomline=false,roundcorner=2mm] 
	\vspace{-.1in}
	\begin{definition}[Kullback--Leibler Leakage]
		The Kullback--Leibler information leakage in machine learning is 
		\begin{align*}
			\rho_{\mathrm{KL}}(\theta^*_\d):=&\E\{
			D_{\mathrm{KL}}(p(\theta^*_\d|x_i,y_i,z_i=1)\| p(\theta^*_\d|x_i,y_i,z_i=0))\\
			&+
			D_{\mathrm{KL}}(p(\theta^*_\d|x_i,y_i,z_i=0)\|p(\theta^*_\d|x_i,y_i,z_i=1))\}.
		\end{align*}
		Again, similarly, we can define $\rho_{\mathrm{KL}}(\theta^*_\c)$. 
	\end{definition}
\end{mdframed}

We readily get a relationship between $\rho_{\mathrm{KL}}(\theta^*_\d)$ and $\rho_{\mathrm{ML}}(\theta^*_\d)$. This is expressed in the next corollary. 

\begin{mdframed}[backgroundcolor=black!10,rightline=false,leftline=false,topline=false,bottomline=false,roundcorner=2mm] 
	\vspace{-.1in}
	\begin{corollary} \label{cor:1} $\rho_{\mathrm{MI}}(\theta^*_\d)\leq \rho_{\mathrm{KL}}(\theta^*_\d)/4$.
	\end{corollary}
\end{mdframed}

We can relate the Kullback--Leibler information leakage to the ability of an adversary inferring whether a data point belongs to the training dataset. This is investigated in the following theorem.

\begin{mdframed}[backgroundcolor=black!10,rightline=false,leftline=false,topline=false,bottomline=false,roundcorner=2mm] 
	\vspace{-.1in}
\begin{theorem} Let $\Psi_{x_i,y_i}:\mathbb{R}^{p_\theta}\rightarrow \{0,1\}$ denote the policy of the adversary for determining whether $(x_i,y_i)$ belongs to the training set based on access to the trained model $\theta^*_\d$. Then 
	\begin{align*}
	\P\{\Psi_{x_i,y_i}(\theta^*_\d)= z_i\}\leq \frac{1}{2}\sqrt{\rho_{\mathrm{KL}}(\theta^*_\d)}.
	\end{align*}
\end{theorem}
\end{mdframed}

\begin{proof}
Using Le Cam's inequality~\cite{Yu1997}, we get
\begin{align*}
\inf_{\Psi_{x_i,y_i}}& \P\{\Psi_{x_i,y_i}(\theta^*_\d)\neq z_i\}=1-\nu(p_1(\theta^*_\d),p_0(\theta^*_\d)),
\end{align*}
where $\nu$ is the total variation distance defined as
\begin{align*}
\nu(p_1(\theta^*_\d),p_0(\theta^*_\d)):=\hspace{-.04in}\sup_{\mathcal{A}\in 2^{\Theta_\d}} \hspace{-.04in}|&\mathbb{P}\{\theta^*_\d\in\mathcal{A}|x_i,y_i,z_i=1\}-\mathbb{P}\{\theta^*_\d\in\mathcal{A}|x_i,y_i,z_i=0\}|.
\end{align*}
Hence,
\begin{align*}
\sup_{\Psi_{x_i,y_i}} \P\{\Psi_{x_i,y_i}(\theta^*_\d)=z_i\}
&=\sup_{\Psi_{x_i,y_i}} [1-\P\{\Psi_{x_i,y_i}(\theta^*_\d)\neq z_i\}]\\
&=1-\inf_{\Psi_{x_i,y_i}}\P\{\Psi_{x_i,y_i}(\theta^*_\d)\neq z_i\}\\
&=\nu(p_1(\theta^*_\d),p_0(\theta^*_\d)).
\end{align*}
Note that
\begin{align*}
\E\{\nu(p_1(\theta^*_\d),p_0(\theta^*_\d))\}
\leq & \E\Bigg\{\hspace{-.04in}\sqrt{
	\frac{1}{2} D_{\mathrm{KL}}(p_1(\theta^*_\d)||p_0(\theta^*_\d))
}\Bigg\}\\
\leq &\sqrt{\hspace{-.04in}\frac{1}{2}\E\Bigg\{\hspace{-.04in}
	D_{\mathrm{KL}}(p_1(\theta^*_\d)||p_0(\theta^*_\d))
	\Bigg\}},
\end{align*}	
where the first inequality follows from the Pinsker's inequality~\cite{massart2007concentration} and the second inequality follows from the Jensen's inequality~\cite{cover2012elements} while noting that $x\mapsto\sqrt{x}$ is a concave function. Similarly, we can prove that
\begin{align*}
\E\{\nu(p_1(\theta^*_\d),p_0(\theta^*_\d))\} 
\leq &\sqrt{\hspace{-.04in}\frac{1}{2}\E\Bigg\{\hspace{-.04in}
	D_{\mathrm{KL}}(p_1(\theta^*_\d)||p_0(\theta^*_\d))
	\Bigg\}}.
\end{align*}	
Combining these two inequalities results in 
\begin{align*}
\E\{\nu(p_1(\theta^*_\d),p_0(\theta^*_\d))\}^2 \hspace{-.03in}
\leq &\frac{1}{2}\E\{
	\min\{D_{\mathrm{KL}}(p_0(\theta^*_\d)||p_1(\theta^*_\d)),D_{\mathrm{KL}}(p_1(\theta^*_\d)||p_0(\theta^*_\d))\}
	\}.
\end{align*}
The rest of the proof follows from that 
\begin{align*}
&\min\{D_{\mathrm{KL}}(p_0(\theta^*_\d)||p_1(\theta^*_\d)),D_{\mathrm{KL}}(p_1(\theta^*_\d)||p_0(\theta^*_\d))\}\leq (D_{\mathrm{KL}}(p_0(\theta^*_\d)||p_1(\theta^*_\d))+D_{\mathrm{KL}}(p_1(\theta^*_\d)||p_0(\theta^*_\d)))/2.
\end{align*}
This concludes the proof.
\end{proof}

\begin{remark}[Black-box vs. White-box]
	In both definitions of the membership information leakage, we assume that the adversary has access to the parameters of the trained model  $\theta^*_\d$ (i.e., white-box assumption). This is the strongest assumption for an adversary and the amount of the leaked information reduces if we instead let the adversary query the model (i.e., black-box assumption). In fact, the data processing inequality states that 
	$I(\mathfrak{M}(x_i;\theta^*_\d),y_i;z_i|x_i,y_i)\leq I(\theta^*_\d;z_i|x_i,y_i)=\rho_{\mathrm{MI}}(\theta^*_\d).$
	We are interested in analyzing this framework as it provides an insight against the worst-case adversary and therefore the mitigation techniques extracted from this analysis would also work against weaker adversaries with more restricted access to the model. 
\end{remark}

\begin{remark}[Gaussian Approximation]
Let us present a simple numerical method for computing Kullback--Leibler information leakage using Gaussian approximation. To this aim, we can approximate $p_1(\theta^*_\c)$ and $p_0(\theta^*_\c)$ by Gaussian density functions $\mathcal{N}(\mu_1^{x_i,y_i},\Sigma_1^{x_i,y_i})$ and $\mathcal{N}(\mu_0^{x_i,y_i},\Sigma_0^{x_i,y_i})$, respectively. The parameters $\mu_0^{x_i,y_i}$, $\mu_1^{x_i,y_i}$, $\Sigma_0^{x_i,y_i}$, and $\Sigma_1^{x_i,y_i}$ are extracted by Monte-Carlo simulation with and without $x_i,y_i$. These distributions are often more complex than Gaussian and are approximated with Gaussian distributions for numerical evaluation. When the underlying distributions are close to Gaussian, the errors in such approximations are inversely proportional to the square root of the number of the Monte Carlo scenarios. Using the Gaussian approximation, we get
$
\varrho_1(x_i,y_i)\hspace{-.03in}:=D_{\mathrm{KL}}(p_1(\theta^*_\c)||p_0(\theta^*_\c))*
=0.5(\trace((\Sigma_0^{x_i,y_i})^{-1}\Sigma_0^{x_i,y_i})-p_\theta
\hspace{-.03in}+\hspace{-.03in}(\mu_0^{x_i,y_i}\hspace{-.03in}-\hspace{-.03in}\mu_1^{x_i,y_i})(\Sigma_0^{x_i,y_i})^{-1}(\mu_0^{x_i,y_i}\hspace{-.03in}-\hspace{-.03in}\mu_1^{x_i,y_i})
\hspace{-.03in}+\hspace{-.03in}\ln(\det(\Sigma_0^{x_i,y_i})/\det(\Sigma_1^{x_i,y_i}))).$
We can similarly evaluate the value of  $\varrho_2(x_i,y_i):=D_{\mathrm{KL}}(p_0(\theta^*_\c)||p_1(\theta^*_\c)).$
Then, we can approximate $\rho_{\mathrm{KL}}(\theta^*_\d)$ by computing $\varrho_1(x_i,y_i)$ and $\varrho_2(x_i,y_i)$ for a set of data entries $\J\subseteq\U$ and compute $\rho_{\mathrm{KL}}(\theta^*_\c)\approx \frac{1}{|\J|} \sum_{(x,y)\in\J} (\varrho_1(x,y)+\varrho_2(x,y)).$
This enables us to approximately compute the Kullback–Leibler information leakage in machine learning. Interestingly, this is an approximation method for computing $\rho_{\mathrm{MI}}(\theta^*_\c)$ when approximating $p_1(\theta^*_\c)$ and $p_0(\theta^*_\c)$ by Gaussian density functions~\cite{4218101}. This is because, under Gaussian approximation, $p(\theta^*_\c|x_i,y_i)$ follows a Gaussian mixture. Furthermore, due to data processing inequality, $\rho_{\mathrm{MI}}(\theta^*_\d)\leq \rho_{\mathrm{MI}}(\theta^*_\c)$. Therefore, in Section~\ref{sec:numerical}, we use this approximation for numerically computing $\rho_{\mathrm{MI}}(\theta^*_\d)$.
\end{remark}

\section{What Influences Membership Inference?}
One of the most important factors in machine learning is the size of the training dataset. In what follows, we show that the success of the membership inference attack is inversely proportional to the size of the training dataset. 

\begin{mdframed}[backgroundcolor=black!10,rightline=false,leftline=false,topline=false,bottomline=false,roundcorner=2mm] 
	\vspace{-.1in}
	\begin{theorem} \label{tho:dataset_size_1} Assume that\vspace{-.15in}
	\begin{itemize}
	\item[(A1)] $\Theta_{\c}$ is compact, $\Theta_{\d}\subset\Theta_\c$, and $|\Theta_{\d}|<\infty$;\vspace{-.1in}
	\item[(A2)] $\{(x_i,y_i)\}_{i=1}^N$ is i.i.d. following distribution $\mathcal{P}$;\vspace{-.1in}
	\item[(A3)] $\lambda g(\theta)+\mathbb{E}_{\mathcal{P}}\{\ell(\mathfrak{M}(x;\theta),y)\}$ is continuous and has a unique minimizer;\vspace{-.1in}
	\item[(A4)] $\ell(\mathfrak{M}(x;\theta),y)$ is almost surely Lipschitz continuous with Lipschitz constant $L(x,y)$ on $\Theta_{\c}$ with respect to $\mathcal{P}$ and $\mathbb{E}_{\mathcal{P}}\{L(x,y)\}<\infty$.\vspace{-.1in}
	\end{itemize}
 Then, $\displaystyle \lim_{n,N\rightarrow\infty:n\leq N} \rho_{\mathrm{MI}}(\theta^*_\d)=0$.
	\end{theorem}
\end{mdframed}

\begin{proof}
Using Proposition~8.5 in~\cite{Kim2015}, (A1), (A2) and (A4) ensures that $f(\theta;\mathcal{D})$ converges to $\lambda g(\theta)+\linebreak \mathbb{E}_{\mathcal{P}}\{\ell(\mathfrak{M}(x;\theta),y)\}$ almost surely uniformly on $\Theta_{\c}$ as $n\leq N$ tends to infinity. Following this observation in conjunction with (A1), (A3) and Theorem~8.2 in~\cite{Kim2015}, we get that $\theta^*_\c$ converges to $\theta'\in\argmin_{\theta\in\Theta_{\c}} \lambda g(\theta)+\mathbb{E}_{\mathcal{P}}\{\ell(\mathfrak{M}(x;\theta),y)\}$ almost surely.  Therefore, $\theta^*_\d$ converges to $\Pi_{\Theta_\d}[\theta']$ almost surely. Almost sure convergence implies convergence in probability and that in turn implies convergence in distribution~\cite[p.\,38]{klebaner2005introduction}. Now, the continuity of mutual information on finite alphabet implies that $\lim_{n,N\rightarrow\infty:n\leq N} \rho_{\mathrm{MI}}(\theta^*_\d)\linebreak = I(\Pi_{\Theta_\d}[\theta'];z_i|x_i,y_i)=0$ because $\Pi_{\Theta_\d}[\theta']$ is deterministic.
\end{proof}

\begin{remark}[Projection vs Discrete Optimization] Instead of using the projection in~\eqref{eqn:ML_d}, we could rewrite the optimization problem for training the machine learning model in~\eqref{eqn:ML} as a discrete optimization problem over decision set $\Theta_\d$. Doing so, we could still prove the result of Theorem~\ref{tho:dataset_size_1} while relying on the properties of sample average approximation in stochastic discrete problems~\cite{kleywegt2002sample}. We opted for the projection-based approach in~\eqref{eqn:ML_d} as it is closer in spirit to the solutions extracted by computers.
\end{remark}

Theorem~\ref{tho:dataset_size_1} states that the \textbf{amount of the membership information leakage is a decreasing function of the size of the training dataset}. This shows that, by using a larger training dataset, the model is less over-fit to the training dataset and it is therefore harder to distinguish the training data. Increasing the dataset size also helps with over-learning (or memorization), which is another possible factor behind success of membership inference attacks. This is in-line with the observation that over-fitting contributes to success of membership inference attacks~\cite{yeom2018privacy}. The result of Theorem~\ref{tho:dataset_size_1} does not require convexity of the loss function or even its differentiability.  Therefore, it is applicable to different machine learning models ranging from linear regression and support vector machines to neural networks and decision trees. In the next theorem, by focusing on convex smooth machine learning problems, we investigate the effect of other factors, such as regularization, on the success of membership inference attacks.

\begin{mdframed}[backgroundcolor=black!10,rightline=false,leftline=false,topline=false,bottomline=false,roundcorner=2mm] 
	\vspace{-.1in}
	\begin{theorem} \label{tho:2} Assume that\vspace{-.15in}
		\begin{itemize}
			\item[(A1)]$\Theta_{\c}$ is compact, $\Theta_{\d}\subset\Theta_\c$, and $|\Theta_{\d}|<\infty$;\vspace{-.1in}
			\item[(A2)] $\lambda g(\theta)+\mathbb{E}_{\mathcal{P}}\{\ell(\mathfrak{M}(x;\theta),y)\}$ is continuous and finite everywhere;\vspace{-.1in}
			\item[(A3)] $\ell(\mathfrak{M}(x;\theta),y)$ is almost surely Lipschitz continuous with Lipschitz constant $L$ on $\Theta_{\c}$;\vspace{-.1in}
			\item[(A4)] $g(\theta)$ is strictly convex and $\mathbb{E}\{\ell(\mathfrak{M}(x;\theta),y)\}$ is convex.\vspace{-.1in}
		\end{itemize}
		Then, \vspace{-.15in}
		\begin{align}
		\lim_{\lambda\rightarrow\infty} \rho_{\mathrm{MI}}(\theta^*_\d)&=0.
		\end{align}
		Consider a family of fitness functions $\ell(\mathfrak{M}(x;\theta),y)$ parameterized by the the Lipschitz constant $L\in[0,c)$ for some $c>0$, then
		\begin{align}
		\lim_{L\rightarrow 0} \rho_{\mathrm{MI}}(\theta^*_\d)&=0.
		\end{align}
	\end{theorem}
\end{mdframed}

\begin{proof} 
	The Maximum Theorem implies that $\theta^*_\c$ is a continuous function of $\lambda$~\cite[p.\,237]{sundaram1996first}. Thus, $\lim_{\lambda\rightarrow\infty} \theta^*_\c=\bar{\theta}:=\argmin_{\theta\in\Theta_{\c}} g(\theta)$ and, as a result, $\lim_{\lambda\rightarrow\infty} \theta^*_\d=\Pi_{\Theta_\d}[\bar{\theta}]$. Thus, 	$\lim_{\lambda\rightarrow\infty}I(\theta^*_\d;z_i|x_i,y_i)= I(\Pi_{\Theta_\d}[\bar{\theta}];z_i|x_i,y_i)\linebreak=0$.  Again, the Maximum Theorem  implies that $\theta^*_\c$ is a continuous function of $L$ over a family of fitness functions $\ell(\mathfrak{M}(x;\theta),y)$ parameterized by the the Lipschitz constant $L$.  For $L=0$, the  fitness function $\ell(\mathfrak{M}(x;\theta),y)$ is independent of $\theta$ because
	$0\leq \|\ell(\mathfrak{M}(x;\theta),y)
	-\ell(\mathfrak{M}(x;\theta'),y)\|\leq L\|\theta-\theta'\|=0$ for all $\theta,\theta'\in\Theta_{\c}$. Thus $\lim_{L\rightarrow 0} \theta^*_\d=\Pi_{\Theta_\d}[\bar{\theta}]$ and, similarly,  $\lim_{L\rightarrow 0}I(\theta^*_\d;z_i|x_i,y_i)=0$.
\end{proof}

Theorem~\ref{tho:2} shows that \textbf{regularization reduces the amount of membership information leakage}. Increasing the importance of the regularization reduces the over-fitting, over-learning, or memorization and is therefore an important tool for combating membership inference attacks. Let us define model sensitivity by
\begin{align*}
S:=\sup_{(x,y)}\sup_{\theta} \left\|\frac{\partial \ell(\mathfrak{M}(x;\theta),y)}{\partial \theta}\right\|.
\end{align*}
Clearly, $L\leq S$. Therefore, Theorem~\ref{tho:2} shows that, \textbf{if the model sensitivity is high, more membership information is potentially leaked}. Therefore, complex models, such as deep neural networks, are more susceptible to membership inference attacks in comparison to simpler models with fewer degrees of freedom. 

\section{Additive Noise for Membership Privacy}
In this section, we explore the use of additive noise, particularly, differential privacy noise, on the amount of the leaked membership information. 

\begin{mdframed}[backgroundcolor=black!10,rightline=false,leftline=false,topline=false,bottomline=false,roundcorner=2mm] 
	\vspace{-.1in}
	\begin{theorem} \label{tho:MI_increase_snr} Assume $w$ is a zero-mean Gaussian variable with unit variance. The mutual membership information leakage is $\rho_{\mathrm{MI}}(\theta^*_\c+tw)$ and Kullback–Leibler membership information leakage $\rho_{\mathrm{KL}}(\theta^*_\c+tw)$ are decreasing functions of $t$. Particularly, $\rho_{\mathrm{MI}}(\theta^*_\c+tw)=\rho_{\mathrm{MI}}(\theta^*_\c)-\mathcal{O}(t)$ and  $\rho_{\mathrm{KL}}(\theta^*_\c+tw)=\rho_{\mathrm{KL}}(\theta^*_\c)-\mathcal{O}(t).$
	\end{theorem}
\end{mdframed}

\begin{proof} Using de~Bruijn Identity in Terms of Divergences~\cite{valero2017generalization}, we get
	\begin{align*}
	&\frac{\mathrm{d}D_{\mathrm{KL}}(p_1(\theta^*_\c+tw)\| p_0(\theta^*_\c+tw))}{\mathrm{d}t}
=\hspace{-.02in}-\frac{J(p_1(\theta^*_\c+tw)\| p_0(\theta^*_\c+tw))}{2},
	\end{align*}
	where $J(p_1(\theta^*_\c+tw)\| p_0(\theta^*_\c+tw))$ is the Fisher divergence defined as
	\begin{align*}
J(p_1(\theta^*_\c+tw)\| p_0(\theta^*_\c+tw))
=&\int \Bigg[\nabla_t \log\left(\frac{p_1(\theta^*_\c+tw)}{p_0(\theta^*_\c+tw)}\right)\Bigg]^\top
 \Bigg[\nabla_t \log\left(\frac{p_1(\theta^*_\c+tw)}{p_0(\theta^*_\c+tw)}\right)\Bigg] p_0(\theta^*_\c+tw)\mathrm{d}\theta^*_\c.
	\end{align*}
	Note that, by construct, $J$ is semi-positive definite. Note that
	$D_{\mathrm{KL}}(p_1(\theta^*_\c+tw)\| p_0(\theta^*_\c+tw))
	=-c't+\mathcal{O}(t^2),$
	and $D_{\mathrm{KL}}(p_0(\theta^*_\c+tw)\| p_1(\theta^*_\c+tw))=-c''t+\mathcal{O}(t^2),$
	where
	$
	c'=(1/2)J(p_1(\theta^*_\c+tw)\| p_0(\theta^*_\c+tw))|_{t=0}
	\geq 0$ and $c''=(1/2)J(p_0(\theta^*_\c+tw)\| p_1(\theta^*_\c+tw))|_{t=0}|_{t=0}\geq 0.$
	Selecting $c_{\mathrm{KL}}=c'+c''$ concludes the proof for the mutual information by proving that $\rho_{\mathrm{KL}}(\theta^*_\c+tw)=\rho_{\mathrm{KL}}(\theta^*_\c)-c_{\mathrm{KL}}t+\mathcal{O}(t^2)$. In light of~\eqref{eqn:MI_reduced_KL}, the proof for the mutual membership information leakage follows the same line of reasoning. 
\end{proof}

Theorem~\ref{tho:MI_increase_snr} proves that \textbf{the amount of  membership information leakage is increased by reducing the amount of the additive noise}. This is why differential privacy works as a successful defence strategy in membership inference~\cite{10shokri2017membership}, albeit if its privacy budget is set small enough. In what follows, we focus on the effect of differential privacy noise on the membership information leakage.

\begin{definition}[Differential Privacy]
Mechanism $\theta^*_\c+w$ with additive noise $w$ is $(\epsilon,\delta)$-differentially private if, for all Lebesgue-measurable sets $\mathcal{W}\subseteq\mathbb{R}^{p_\theta}$,
\begin{align*}
\P\{\theta^*_\c+w\in&\mathcal{W}|(z_i)_1^N\}\leq \exp(\epsilon)\P\{\theta^*_\c+w\in\mathcal{W}|(z'_i)_1^N\}+\delta ,
\end{align*}
where $(z_i)_1^N$ and $(z'_i)_1^N$ are any two vectors in $\{0,1\}^N$ such that $\sum_{i=1}^N z_i=n$, $\sum_{i=1}^N z'_i=n$, and there exists at most one index $j$ for which $z_j\neq z'_j$. 
\end{definition}

It has been proved that we can ensure $(\epsilon,\delta)$-differential privacy with Gaussian noise. This is recited in the following proposition. 

\begin{proposition} \label{prop:differential_privacy}
Assume that $\Delta \theta^*_\c>0$ is such that $\|\theta^*_\c((z_i)_1^N)-\theta^*_\c((z'_i)_1^N)\|_2\leq \Delta \theta^*_\c$, 
where $(z_i)_1^N$ and $(z'_i)_1^N$ are any two vectors in $\{0,1\}^N$ such that $\sum_{i=1}^N z_i=n$, $\sum_{i=1}^N z'_i=n$, and there exists at most one index $j$ for which $z_j\neq z'_j$.  Then, the mechanism $\theta^*_\c+w$ is $(\epsilon,\delta)$-differentially private if $w$ is a zero-mean Gaussian noise with standard deviation $\sigma=\sqrt{2\log(1.25/\delta)}\Delta \theta^*_\c/\epsilon$.
\end{proposition}

\begin{proof}
The proof immediately follows from using the Gaussian mechanism for ensuring differential privacy; see, e.g.,~\cite[Theorem~A.1]{dwork2014algorithmic}.
\end{proof}

\begin{mdframed}[backgroundcolor=black!10,rightline=false,leftline=false,topline=false,bottomline=false,roundcorner=2mm] 
	\vspace{-.1in}
\begin{corollary} \label{cor:dp} Assume that $w$ is selected to ensure $(\epsilon,\delta)$-differential privacy based on Proposition~\ref{prop:differential_privacy}. Then, $\rho_{\mathrm{MI}}(\theta^*_\c+w)=\rho_{\mathrm{MI}}(\theta^*_\c)-\mathcal{O}(\log^{1/2}(\delta^{-1})\epsilon^{-1})$ and $\rho_{\mathrm{KL}}(\theta^*_\c+w)=\rho_{\mathrm{KL}}(\theta^*_\c)-\mathcal{O}(\log^{1/2}(\delta^{-1})\epsilon^{-1})$.
\end{corollary}
\end{mdframed}

Again, noting the data processing inequality, we get $\rho_{\mathrm{MI}}(\theta^*_\d)\leq \rho_{\mathrm{MI}}(\theta^*_\c)$. Therefore, we expect $\rho_{\mathrm{MI}}(\theta^*_\d)$ to follow the trend as in Corollary~\ref{cor:dp}.

\section{Experimental Validation} \label{sec:numerical}
In this section, we demonstrate the results of the paper numerically using a practical dataset. 
\begin{figure}[t]
	\centering 
	\begin{tikzpicture}
	\node[] at (0,0) {\includegraphics[width=.45\columnwidth]{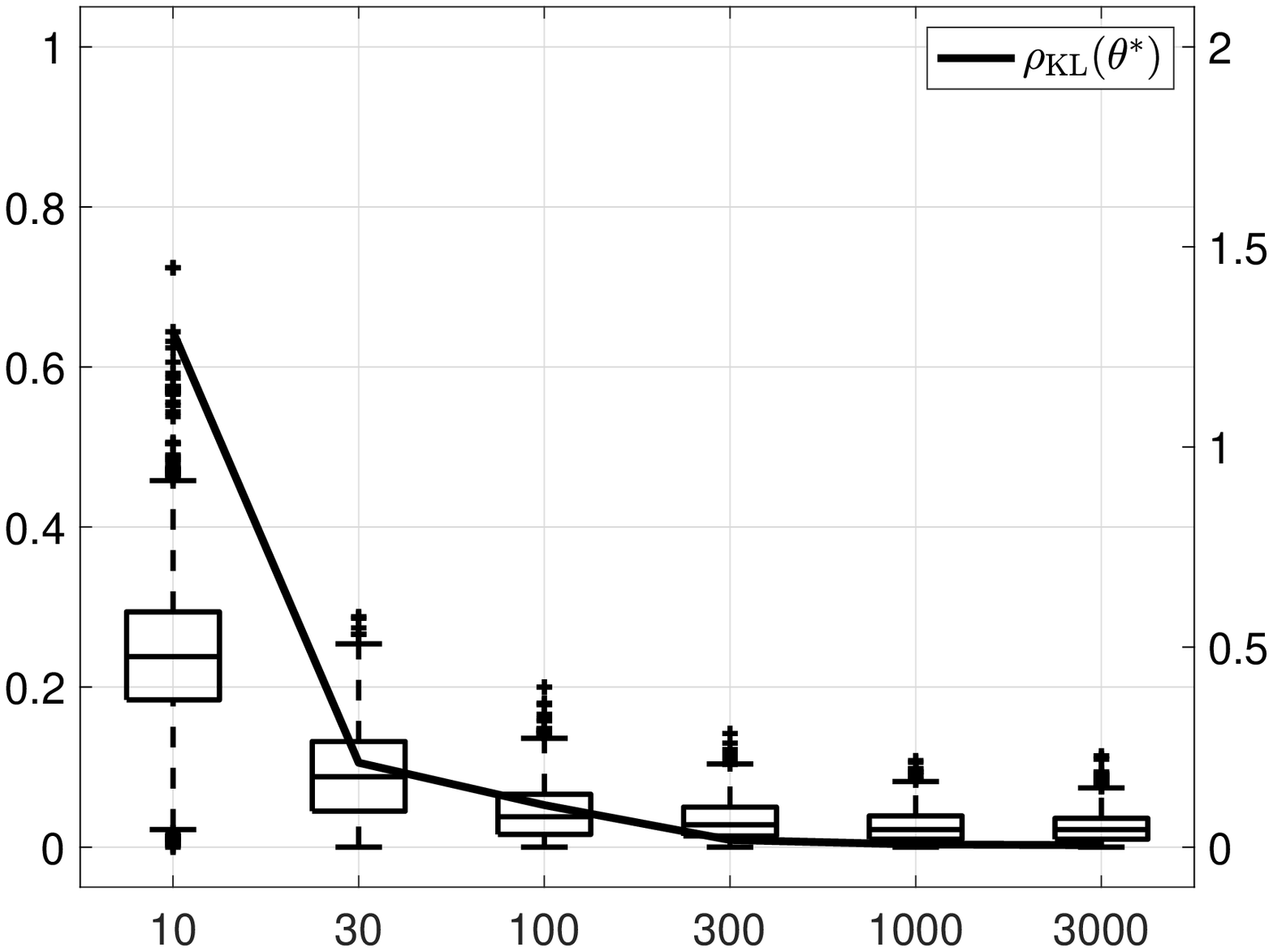}};
	\node[] at (0,-2.9) {$n$};
	\node[rotate=90] at (-3.7,0) {$\mathrm{Adv}$};
	\node[rotate=90] at (+3.7,0) {$\rho_{\mathrm{KL}}(\theta^*_\c)$};
	\end{tikzpicture}
	\vspace{-.1in}
	\caption{\label{fig:1} The relationship between adversary's advantage in membership attack and the size of the training dataset for linear regression ($\lambda=0$ and $p_x=5$). 
	}
\vspace{.1in}
	\begin{tikzpicture}
	\node[] at (0,0) {\includegraphics[width=.45\columnwidth]{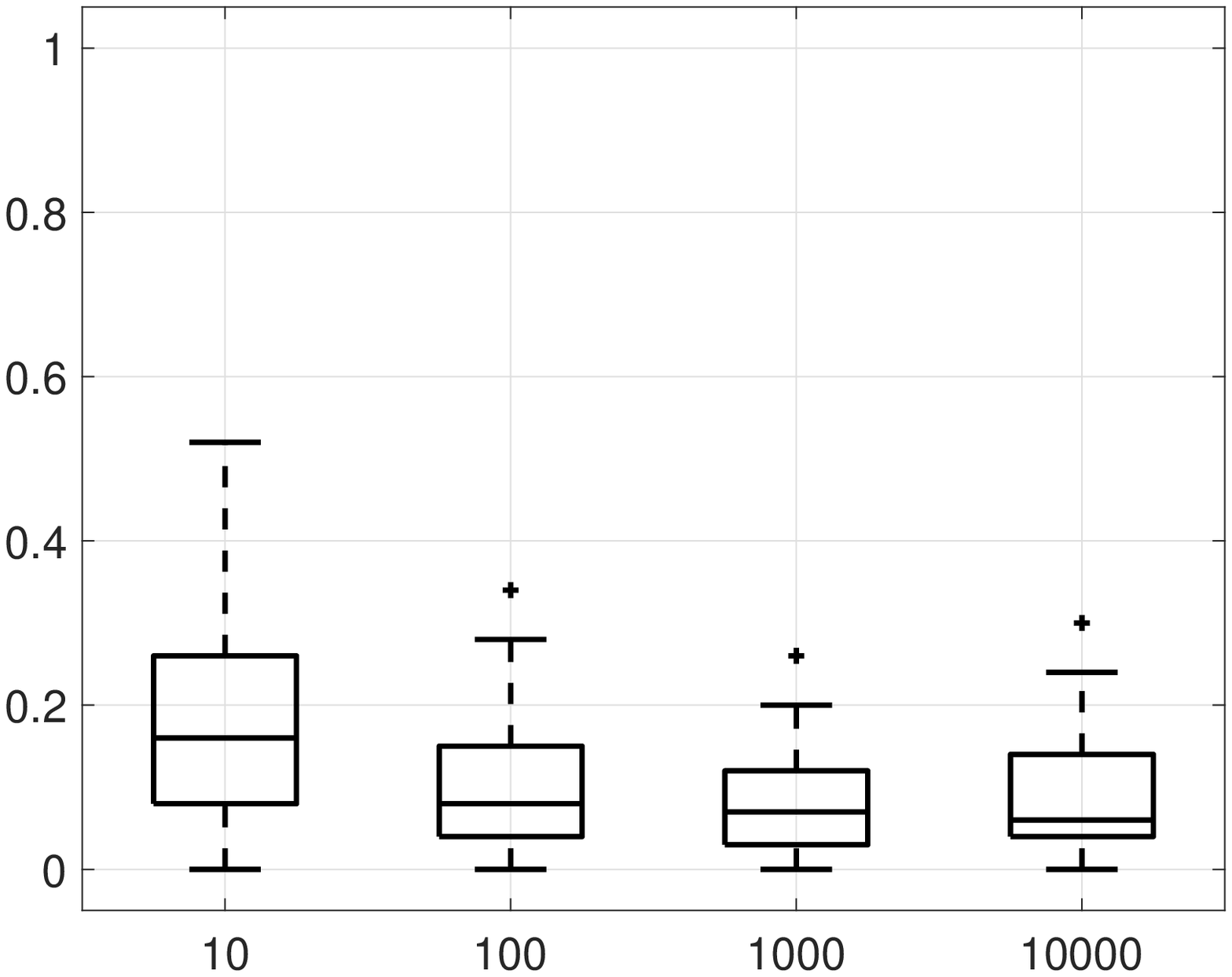}};
	\node[] at (0,-2.9) {$n$};
	\node[rotate=90] at (-3.7,0) {$\mathrm{Adv}$};
	\end{tikzpicture}
	\vspace{-.1in}
	\caption{\label{fig:11} The relationship between adversary's advantage in membership attack versus the size of the training dataset for neural network ($\lambda=0$ and $p_x=5$). 
	}
\end{figure}

\begin{figure}[t]
	\centering
	\begin{tikzpicture}
	\node[] at (0,0) {\includegraphics[width=.45\columnwidth]{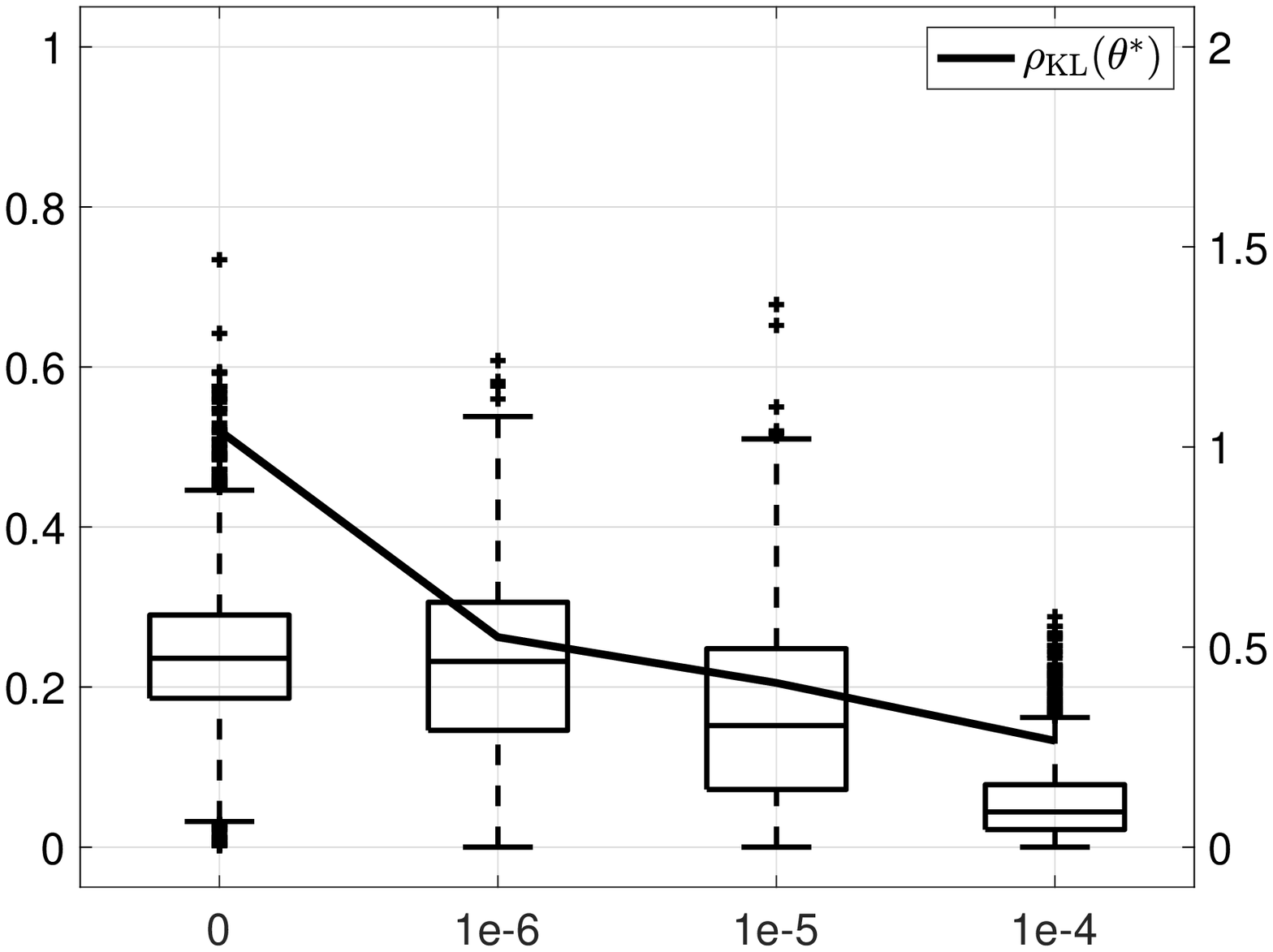}};
	\node[] at (0,-2.9) {$\lambda$};
	\node[rotate=90] at (-3.7,0) {$\mathrm{Adv}$};
	\node[rotate=90] at (+3.7,0) {$\rho_{\mathrm{KL}}(\theta^*_\c)$};	
	\end{tikzpicture}
	\vspace{-.1in}
	\caption{\label{fig:3} The relationship between adversary's advantage in membership attack and regularization weight $\lambda$ for linear regression ($n=10$ and $p_x=5$).
	}
\vspace{.1in}
	\begin{tikzpicture}
	\node[] at (0,0) {\includegraphics[width=.45\columnwidth]{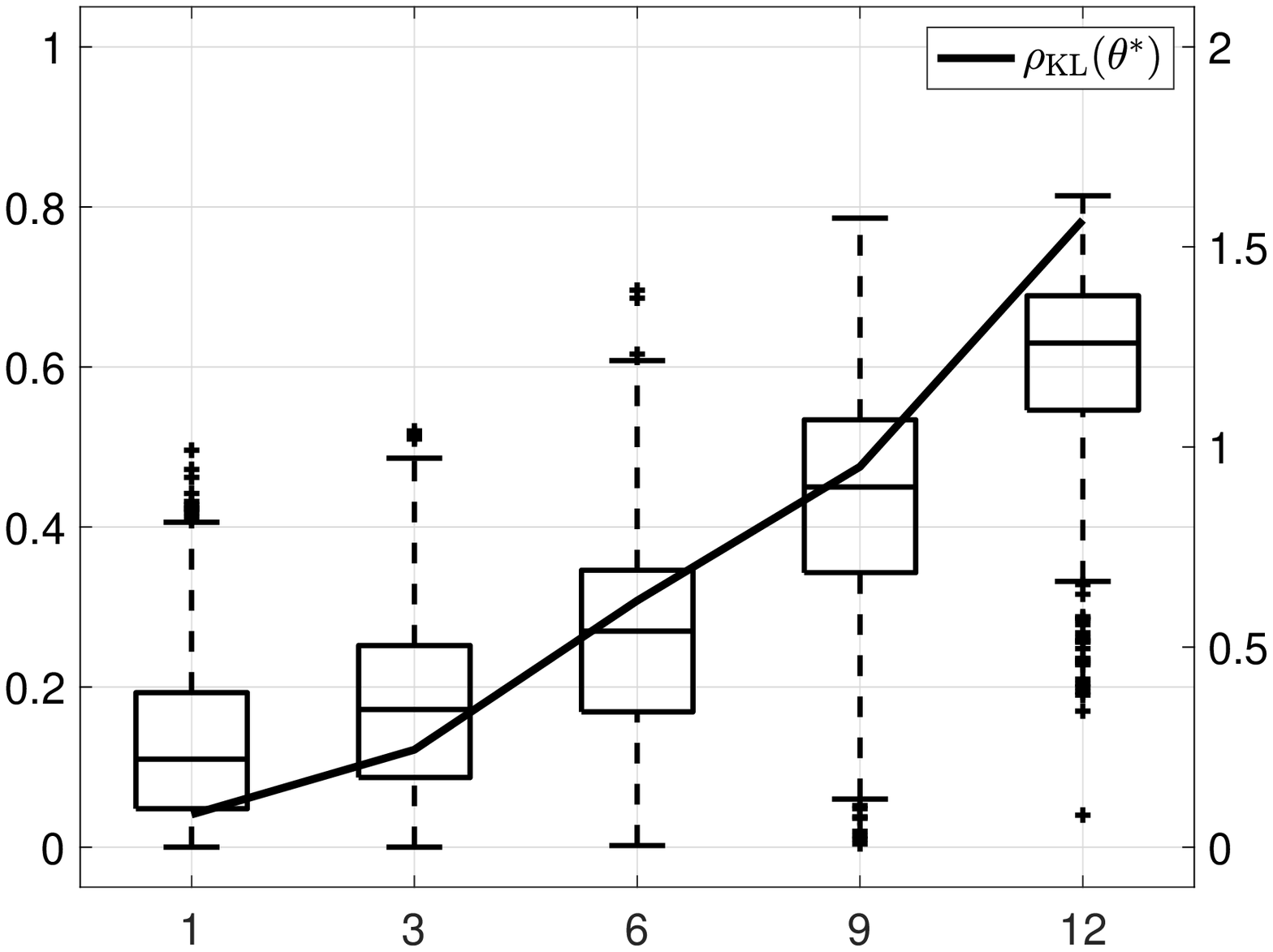}};
	\node[] at (0,-2.9) {$p_x$};
	\node[rotate=90] at (-3.7,0) {$\mathrm{Adv}$};
	\node[rotate=90] at (+3.7,0) {$\rho_{\mathrm{KL}}(\theta^*_\c)$};	
	\end{tikzpicture}
	\vspace{-.1in}
	\caption{\label{fig:p} The relationship between adversary's advantage in membership attack versus number of features $p_x$ for linear regression ($n=10$ and $\lambda=10^{-6}$).
	}
\end{figure}

\begin{figure}[t]
	\centering
	\begin{tikzpicture}
	\node[] at (0,0) {\includegraphics[width=.45\columnwidth]{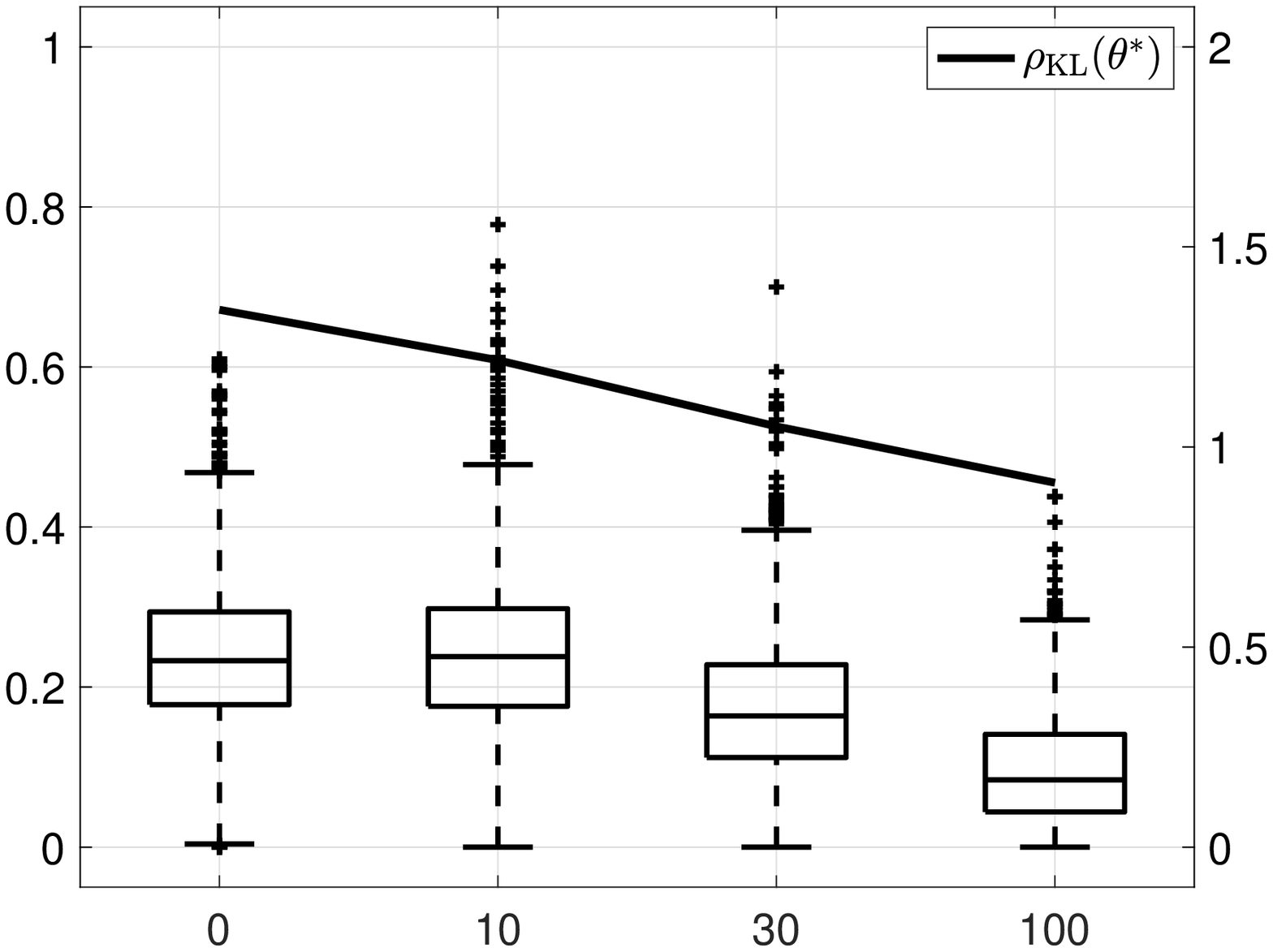}};
	\node[] at (0,-2.9) {$\sigma$};
	\node[rotate=90] at (-3.7,0) {$\mathrm{Adv}$};
	\node[rotate=90] at (+3.7,0) {$\rho_{\mathrm{KL}}(\theta^*_\c)$};	
	\end{tikzpicture}
\vspace{-.1in}
\caption{\label{fig:4} The relationship between adversary's advantage in membership attack and the standard deviation of additive noise $\sigma$ for linear regression ($n=30$, $\lambda=0$,  and $p_x=5$). 
}
\end{figure}

\subsection{Dataset Description}
We use the Adult Dataset available from the UCI Machine Learning Repository~\cite{Dua2019}. This dataset was first used in~\cite{kohavi1996scaling}. The Adult Dataset contains 14 individual attributes, such as age, race, occupation, and relationship status, as inputs and income level (i.e., above or below \$50K per annum) as output. The dataset contains $N=48,842$ instances extracted from the  1994 Census database. We translate all categorical attributes and outputs to integers. We perform feature selection using the Principal Component Analysis (PCA) to select the top $p_x$ important features. This greatly improves the numerical stability of the underlying machine learning algorithms. In what follows, we select $p_x=5$ except for one example in which we vary $p_x$ to study its effect on the success of membership inference and membership information leakage. 

\subsection{Experiment Setup}
We use linear regression to demonstrate all the results of the paper, namely, Theorems~\ref{tho:dataset_size_1}--\ref{tho:MI_increase_snr}. However, noting that the results of Theorem~\ref{tho:dataset_size_1} also hold for non-convex learning problems, we demonstrate these results also for neural networks with five hidden layers with fifty neurons in each layer and hyperbolic tangent sigmoid activation function. We employ the quadratic regularization function  $g(\theta)=\theta^\top \theta$ if one is used.

For membership inference, we use the threshold-based adversary in~\cite{yeom2018privacy}. To assess the effectiveness of the membership attacks, we also use the membership experiment from~\cite{yeom2018privacy}. Let us describe this experiment briefly. For any $n$, we select $(z_i)_{i=1}^N\in\{0,1\}^N$ uniformly at random such that $\sum_{i=1}^N z_i=n$. We train a model $\theta^*_\d$ based on the training dataset $\mathcal{D}=\{(x_i,y_i)\}_{i:z_i=1}$. We select $b$ with equal probability from $\{0,1\}$. Then, we select a single record from the training dataset $\mathcal{D}=\{(x_i,y_i)\}_{i:z_i=1}$ if $b=1$ or from the remaining data the training dataset $\mathcal{U}\setminus\mathcal{D}=\{(x_i,y_i)\}_{i:z_i=0}$ if $b=0$. We transmit the selected record to the adversary. The adversary estimates the realization of random variable $b$, denoted by $\hat{b}\in\{0,1\}$, based on the selected record and trained model. The adversary's advantage (in comparison to randomly selecting an estimate) is given by
$\mathrm{Adv}:=|2\P\{\hat{b}=b\}-1|. $
We investigate the relationship between this advantage and other factors, such as membership information leakage, training dataset size, regularization, and additive privacy-preserving noise, in the remainder of this section. 

\subsection{Experimental Results}
We first show that the membership information leakage gets smaller by increasing the size of the training dataset; thus validating the prediction of Theorem~\ref{tho:dataset_size_1}. Figure~\ref{fig:1} illustrates the adversary's advantage in membership attack (left axis) and the membership information leakage (right axis) versus the size of the training dataset in the case of linear regression. The box plot shows the adversary's advantage and the solid line shows the membership information leakage. As expected from the theorem, the adversary's advantage in membership attack and the membership information leakage decrease rapidly by increasing the size of the training dataset. Noting that the results of Theorem~\ref{tho:dataset_size_1} also hold for non-convex learning problems, we demonstrate this results also for neural networks. Figure~\ref{fig:11} illustrates the relationship between adversary's advantage in membership attack and the size of the training dataset for a neural network with five hidden layers with fifty neurons in each layer and hyperbolic tangent sigmoid activation function. The same trend is also visible in this case. 

Now, we proceed to validate the prediction of Theorem~\ref{tho:2} regarding the effect of regularization. Figure~\ref{fig:3} shows the adversary's advantage in membership attack (left axis) and the membership information leakage (right axis) versus the regularization weight $\lambda$ for linear regression. Evidently, by increasing the weight of regularization, the adversary's advantage in membership attack and the membership information leakage both decrease. This is intuitive as we expect, by increasing $\lambda$, the trained model to become closer to $\argmin_{\theta\in\Theta_{\c}} g(\theta)$ which is data independent.

An important factor in membership inference success is the number of features.  Figure~\ref{fig:p} illustrates the adversary's advantage in membership attack (left axis) and the membership information leakage (right axis) versus the number of the features extracted from the PCA $p_x$ for linear regression. Now, by increasing the the number of the features, the adversary's advantage in membership attack and the membership information leakage both increase (as more features potentially makes the data entries more unique and their effect of the trained model more pronounced). 

Finally, we investigate the effect of the additive noise on membership inference. Figure~\ref{fig:4} shows the adversary's advantage in membership attack (left axis) and the membership information leakage (right axis) versus the standard deviation of the additive noise to the model parameters for linear regression. By increasing  the magnitude of the noise, the adversary's advantage in membership attack and the membership information leakage both decrease. This is in-line with our theoretical observations from Theorem~\ref{tho:MI_increase_snr}. 

\section{Conclusions and Future Work}
We used mutual information and Kullback--Leibler divergence to develop measures for membership information leakage in machine learning.  We showed that the amount of the membership information leakage is a decreasing function of the training dataset size, the regularization weight, the sensitivity of machine learning model. We also investigated the effect of privacy-preserving additive noise on membership information leakage and the success of membership inference attacks. Future work can focus on further experimental validation of the relationship between the membership information leakage and the success of membership inference attacks for more general machine learning models, e.g., deep neural networks. Another interesting avenue for future research is to use the developed measures of membership information leakage as a regularizer for training machine learning models in order to effectively combat membership inference attacks.

\bibliographystyle{ieeetr}
\bibliography{citation}
\end{document}